\documentclass{l4dc2024}

\def\isarxiv{true}
\newcommand{\version}[1]{\ifthenelse{\equal{\isarxiv}{true}}{{#1}}{\cite{lübsen2024safe}}}

\title[Towards Safe Multi-Task Bayesian Optimization]{Towards Safe Multi-Task Bayesian Optimization}

\usepackage{times}
\usepackage{xcolor}
\usepackage{scalerel,stackengine,amsfonts,mathtools,bm}
\usepackage{dsfont}

\usepackage[capitalize,noabbrev]{cleveref}

\usepackage{times}
\usepackage{siunitx}

\newcommand{\diff}{\mathrm{d}}
\newcommand{\norm}[1]{\left\lVert#1\right\rVert}
\newcommand{\bx}{\bm{x}}
\renewcommand{\footnotesize}{\scriptsize}
\newcommand\equalhat{\mathrel{\stackon[1.5pt]{$=$}{\stretchto{%
    \scalerel*[\widthof{$=$}]{\wedge}{\rule{1ex}{3ex}}}{0.5ex}}}}
\newcolumntype{L}{>{$}l<{$}} 
\newcolumntype{C}{>{$}c<{$}} 

\newtheorem{assum}[theorem]{Assumption}

\coltauthor{
    \Name{Jannis O. Lübsen}$^1$ \Email{jannis.luebsen@tuhh.de} \\
    \Name{Christian Hespe}$^1$ \Email{christian.hespe@tuhh.de}\\
    \Name{Annika Eichler}$^{1,2}$ \Email{annika.eichler@tuhh.de}\\
    \addr $^1$Hamburg University of Technology, Germany\\
    \addr $^2$Deutsches Elektronen-Synchrotron DESY, Germany\\
}

\begin{document}
\maketitle
\begin{abstract}%
    Bayesian optimization has emerged as a highly effective tool for the safe online optimization of systems, due to its high sample efficiency and noise robustness. To further enhance its efficiency, reduced physical models of the system can be incorporated into the optimization process, accelerating it. These models are able to offer an approximation of the actual system, and evaluating them is significantly cheaper. The similarity between the model and reality is represented by additional hyperparameters, which are learned within the optimization process. Safety is a crucial criterion for online optimization methods such as Bayesian optimization, which has been addressed by recent works that provide safety guarantees under the assumption of known hyperparameters. In practice, however, this does not apply. Therefore, we extend the robust Gaussian process uniform error bounds to meet the multi-task setting, which involves the calculation of a confidence region from the hyperparameter posterior distribution utilizing Markov chain Monte Carlo methods. Subsequently, the robust safety bounds are employed to facilitate the safe optimization of the system, while incorporating measurements of the models. Simulation results indicate that the optimization can be significantly accelerated for expensive to evaluate functions in comparison to other state-of-the-art safe Bayesian optimization methods, contingent on the fidelity of the models. The code is accessible on GitHub\footnote{\url{https://github.com/TUHH-ICS/2024-code-L4DC-Safe-Multi-Task-Bayesian-Optimization}}.
\end{abstract}


\section{Introduction}
Bayesian optimization (BO) is an iterative, learning-based, gradient-free, and global optimization method which gained attention in recent years. The method involves learning a probabilistic surrogate model of an arbitrary objective function in order to optimize it, while requiring minor assumptions. The utilization of Gaussian processes (GP) allows including prior knowledge of the objective which makes the procedure very suitable for expensive to evaluate functions and robust to noisy observations. 
Enhanced sample efficiency can be achieved by incorporating reduced models of the objective function into the optimization process, as demonstrated in a proof-of-principle study by \cite{Pousa2023,Letham2019}. The key tool, multi-task Gaussian process prediction, was initially presented in \cite{Bonilla2007}. This approach utilizes correlation matrices to depict the influence between various tasks, learned from available data. Subsequently, \cite{Swersky2013} introduced the first multi-task Bayesian optimization algorithm, where its superior efficiency was highlighted. The key idea revolves around incorporating extra models of the actual function, allowing the prediction of the real function (main task) to be estimated by evaluating the model functions (supplementary tasks) when there exists some correlation. Since evaluating the model functions is significantly more cost-effective in practice, the optimization process can be accelerated considerably, depending upon the fidelity of the supplementary tasks.

Many real-world optimization problems require the consideration of constraints to avoid the evaluation of inputs which invoke undesirable effects, e.g., damaging the system. Often, these constraints are also unknown and need to be learned online. The theoretical fundament for safe Bayesian optimization was inherited from bounding regrets via uniform error bounds in a multi armed bandit problem \cite{Snirivas2010} and improved with respect to performance by \cite{chowdhury17a}. Subsequently, \cite{Sui2015} used the results from \cite{Snirivas2010} to describe \texttt{SafeOpt}, the first safe Bayesian optimization method. Uniform error bounds are defined by scaling the posterior standard deviation by a constant. Typically, these bounds are of probabilistic nature, meaning they hold with high probability. In the previous mentioned works, the assumption is made that the unknown function is deterministic and belongs to the reproducing kernel Hilbert space (RKHS) defined by chosen kernel. Moreover, several additional methodologies have been developed such as the work by \cite{Lederer2019, Sun2021} where the unknown function is assumed to be a sample of the prior distribution defined by the Gaussian process. In statistics, the former approaches are known to operate in the frequentist setting whereas the latter operate in the Bayesian setting, see \cite{murphy2021} for more information about frequentist and Bayesian statistics. However, all these approaches assume that the complexity of the Gaussian process aligns perfectly with the objective, implying that the kernel hyperparameters are known. In case of the RKHS methods, the situation is even worse since they require also knowledge about an upper bound of the RKHS norm of the unknown function. Clearly, from a practical standpoint, this is usually not the case. Furthermore, this aspect is particularly crucial, especially in the context of a multi-task setting, wherein the main task is biased by the supplementary tasks. \cite{Capone2022} introduced a method to address this issue by establishing robust bounds on hyperparameters through a Bayesian framework (c.f. \cref{subsec:Learninghyps}), and \cite{fiedler21} proposed a robustness approach for the frequentist setting where the upper bound on the RKHS norm is still valid. However, these results are not applicable to a multi-task setting.

\paragraph{Contribution} The primary contribution involves introduction of a Bayesian optimization algorithm designed to ensure safe optimization of a system while incorporating measurements from various tasks. This is achieved by extending the outcomes of \cite{Capone2022} to a multi-task setting with the use of Lemma~\ref{lemma:gamma} and \ref{lemma:lambda}. Furthermore, we assume to operate in the Bayesian setting. To the best of the authors' knowledge, this is the first robustly safe Bayesian optimization algorithm in a multi-task setting. Finally, we underscore the significance of the proposal by conducting a benchmark comparison with other state-of-the-art Bayesian optimization methods.

\section{Fundamentals}
In Bayesian optimization, Gaussian processes are used to model an unknown objective function \(f:\mathcal{X} \longrightarrow \mathbb{R}\), where the domain \(\mathcal{X} \subseteq \mathbb{R}^d\) is a compact set of input parameters. In general, \(f\) is non-convex and learned online by evaluating the function at some inputs \( \bx\in\mathcal{X}\). The function values themselves are not accessible, rather noisy observations are made. This behavior can be modeled by additive Gaussian noise \(\epsilon \sim \mathcal{N}(0,\sigma_n^2)\), i.e., $y = f(\bx)+\epsilon$, where $y$ is the measured value and \(\sigma_n^2\) denotes the noise variance.
Furthermore, we assume to have a set of observations by \(\mathcal{D}\coloneqq\left\{(\bx_k,y_k), k=1,\dots, n\right\}\) which is composed of the evaluated inputs combined with the corresponding observations. This set can be considered as the training set. A more compact notation of all inputs of \(\mathcal{D}\) is given by the matrix \(X = [\bx_{1}, \dots, \bx_{n}]^T\) and of all observations by the vector \(\bm{y}=[y_1,\dots,y_n]^T\). With this data set, the Gaussian process creates a probabilistic model to predict $f(\bx)$.  These predictions serve as inputs for an acquisition function $\alpha(\cdot)$, which identifies new promising inputs likely to minimize the objective.
A common choice is expected improvement (EI) as initially described by \cite{Jones1998}. 

\subsection{Gaussian Processes}
\label{subsec:gaussian_processes}
A Gaussian process is fully defined by a mean function \(m(\bx)\) and a kernel \(k(\bx,\bx')\).
The function values $f(\bx)$ are modeled by normal distributions and the kernel determines the dependency between function values at different inputs
\begin{align}
\mathrm{cov}(f(\bx),f(\bx'))=k(\bx,\bx').
\label{eq:single_task_cov}
\end{align} Commonly used kernels are the spectral mixture, Mat\'ern or squared exponential kernel, where the latter is defined as $k_\mathrm{SE}(\bx,\bx') = \sigma_f^2\exp\left(-\frac{1}{2}(\bx-\bx')^T\Delta^{-2}(\bx-\bx')\right)$ with \(\Delta = \mathrm{diag}(\bm{\vartheta}) = \mathrm{diag}([\vartheta_1,\dots,\vartheta_d]^T)\).
The signal variance \(\sigma_f^2\), the lengthscales \(\bm{\vartheta}\) and the noise variance \(\sigma_n^2\) constitute the hyperparameters, allowing for the adjustment of the kernel function. For a more compact notation, we define the tuple of hyperparameters by $\bm{\theta} = (\sigma_f^2,\bm{\vartheta},\sigma_n^2)$.

Given the set of observations, the Gaussian process is used to predict the function values \(f_* \equalhat f(\bx_*)\) at unobserved test points \(\bx_*\in\mathcal{X}\) by determining the posterior distribution $p(f_*|X,\bm{y}) = \mathcal{N}(\mu(\bx_*),\sigma^2(\bx_*))$. 
As shown by \cite{Williams2006} the posterior distribution is Gaussian given by
\begin{align*}
        \mu_{\bm{\theta}}(\bx_*) &= K_{\bm{\theta}}(\bx_*,X)\left(K_{\bm{\theta}}+\sigma_n^2I\right)^{-1}\bm{y}\\
        \sigma^2_{\bm{\theta}}(\bx_*) &= K_{\bm{\theta},*}-K_{\bm{\theta}}(\bx_*,X)\left(K_{\bm{\theta}}+\sigma_n^2I\right)^{-1}K_{\bm{\theta}}(X,\bx_*),
\end{align*}
where $K_{\bm{\theta}} = K_{\bm{\theta}}(X,X)$ and $K_{\bm{\theta},*} = K_{\bm{\theta}}(\bx_*,\bx_*)$ are the Gram matrices of the training and test data, respectively. We use the subscript notation to

\subsubsection*{Learning Hyperparameters}
\label{subsec:Learninghyps}
The selection of hyperparameters is a challenging task due to their substantial impact on the predictive distribution. If selected inaccurately, the predictions do not agree with the true function and safeness conditions may not hold. The most common method is to select the hyperparameters such that the log marginal likelihood \eqref{eq:logmaglikelihood} is maximized, which provides a good trade-off between function complexity and accuracy with respect to the data, given by
\begin{align}
    \log p(\bm{y}|X,\bm{\theta}) = - \frac{1}{2}\log|\tilde{K}_{\bm{\theta}}|-\frac{N}{2}\log(2\pi)-\frac{1}{2}\bm{y}^T\tilde{K}_{\bm{\theta}}^{-1}\bm{y}
    \label{eq:logmaglikelihood}
\end{align}
 with $\tilde{K}_{\bm{\theta}}=K_{\bm{\theta}}+\sigma_n^2I$. The maximization can be carried out efficiently using gradient based methods. However, the log marginal likelihood is non-convex in most cases which makes global optimization challenging. It is commonly assumed that the initial hyperparameters are proximate to the global optimum. In scenarios with limited information about the underlying objective, selecting suitable hyperparameters becomes exceptionally difficult. Consequently, the optimization process may become trapped in a local optimum \cite{Williams2006}.

An advanced strategy was introduced by \cite{Capone2022}. Here, the authors consider the lengthscales in a single-task setting and perform model selection in a Bayesian sense. The idea is to replace the deterministically chosen initial lengthscales $\bm{\vartheta}\in\mathbb{R}^d$ by a prior distribution $p(\bm{\vartheta})$. Then, for a given data set $\mathcal{D}=\{X,\bm{y}\}$ the posterior distribution of the lengthscales can be computed by applying Bayes rule
\begin{align}
    p(\bm{\vartheta}|X,\bm{y}) = \frac{p(\bm{y}|X,\bm{\vartheta})p(\bm{\vartheta})}{p(\bm{y}|X)}.
    \label{eq:posterior_hyperparameters}
\end{align}
Note that computing \eqref{eq:posterior_hyperparameters} analytically is not tractable in general, however, one can use approximations such as Laplace approximation or Markov chain Monte Carlo (MCMC) methods to estimate the posterior. 
Considering the posterior distribution rather than deterministic values offers the opportunity to establish robust probability bounds wherein the majority of the probability mass lies, i.e.,
\begin{align}
    P_\delta=\Bigg\{(\bm{\vartheta}',\bm{\vartheta}'') \in \Theta^2\Bigg\vert\int_{\bm{\vartheta}'\leq\bm{\vartheta}\leq\bm{\vartheta}''} p(\bm{\vartheta}|X,\bm{y})\diff\bm{\vartheta} \geq 1-\delta\Bigg\},
\label{eq:lengthscale_bounds}
\end{align}
where $1-\delta$ is the confidence interval with $\delta\in(0,1)$.
After estimating the ranges, one can apply Lemma~3.3 and Theorem~3.5 from \cite{Capone2022} to obtain the new scaling factor 
\begin{align}
\bar{\beta} = \gamma^2\left(\max_{\bm{\vartheta}'\leq\bm{\vartheta}\leq\bm{\vartheta}''}\beta^\frac{1}{2}(\bm{\vartheta})+\frac{2\norm{y}_2}{\sigma_n}\right)^2,
\end{align}
with $\gamma^2 = \prod_{i=1}^d\frac{\vartheta_i''}{\vartheta_i'}$, that guarantees with probability $(1-\delta)(1-\rho)$ that
$|f(\bx)-\mu_{\bm{\vartheta}_0}(\bx)| \leq \bar{\beta}^\frac{1}{2}\sigma_{\bm{\vartheta}'}(\bx), \forall \bx \in\mathcal{X},$
where $\rho>0$ denotes the failure probability.

\subsection{Multi-Task Gaussian Processes}
So far, only scalar functions are considered, which is different from the considered setting. Since we want to include simulator observations, the GP needs to be extended to model vector-valued functions $\bm{f} = [f_1,\dots f_u]: \mathcal{X} \rightarrow \mathbb{R}^u$.  
To tackle inter-function correlations, the covariance function from \eqref{eq:single_task_cov} is extended by additional task inputs, i.e., $k((\bx,t),(\bx',t'))=\mathrm{cov}(f_t(\bx),f_{t'}(\bx'))$ where $t,t'\in\{1,\dots,u\}$ denote the task indices and $u$ is the total number of information sources. If the kernel can be separated into a task-depending $k_t$ and an input-depending kernel $k_x$, i.e., $k((\bx,t),(\bx',t')) = k_t(t,t')\, k_x(\bx,\bx')$, then the kernel is called separable which is used in most literature, e.g., \cite{Bonilla2007,Letham2019,Swersky2013}. Typically, this is represented by the intrinsic coregionalization model (ICM) \cite{Alvarez2011}, in which $k_x$ measures the dependency in the input space, while $k_t$ measures the dependency between tasks.

Since the task indices are integers and independent of the inputs, the task covariance function is usually substituted by constants $\sigma_{t,t'}^2$ which are treated as additional hyperparameters. Hence, defining the correlation matrix $\Sigma = [\sigma_{t,t'}]_{t,t'=1}^u$, we have $\Sigma\,k(\bx,\bx') = \mathrm{cov}(\bm{f}(\bx),\bm{f}(\bx'))$ which denotes the multi-task kernel. In this setting, it is reasonable to assume positive correlation between tasks solely.
Moreover, it is assumed that for each task there exists a data set $\mathcal{D}_i$ which is stacked into a global set $\bm{\mathcal{D}} \coloneqq \{\bm{X},\tilde{\bm{y}}\}$, where $\bm{X} = [X_1^T,\dots,X_u^T]^T$ and $\tilde{\bm{y}} = [\bm{y}_1^T,\dots,\bm{y}_u^T]^T$. Then, the Gram matrix is given by
\begin{align*}
\bm{K}_\Sigma = \bm{K}_\Sigma(\bm{X},\bm{X})= 
\begin{bmatrix} 
\sigma_{1,1}^2K_{1,1} & \dots & \sigma_{1,u}^2K_{1,u}\\
\vdots & \ddots & \vdots \\
\sigma_{u,1}^2K_{u,1} & \dots & \sigma_{u,u}^2K_{u,u}
\end{bmatrix},    
\end{align*}

where $K_{t,t'}$ are Gram matrices using data sources $t$ and $t'$. For notational reasons, we neglect $\bm{\theta}$ to denote the dependency of the Gram matrices on the remaining hyperparameters, which are equivalent for all tasks.
Note that if the covariance entries are zero, i.e., $\sigma_{t,t'} = 0, \forall t \neq t'$, the off-diagonal blocks of $\bm{K}_\Sigma$ are zero which means that all information sources are independent and can be divided into $u$ separate Gaussian processes. 
To perform inference in the multi-task setting, one simply needs to substitute the single-task Gram matrix $K$ and measurements $\bm{y}$ by their multi-task equivalents $\bm{K}_\Sigma$ and $\tilde{\bm{y}}$. 

\section{Extension of Uniform Error Bounds for Unknown Source Correlation}
Now, we introduce an extension of the uniform error bounds for unknown hyperparameters. We consider the correlation matrix $\Sigma\in \mathcal{C}$ to be the only uncertain hyperparameter and $\mathcal{C}\subset\mathbb{R}^{u\times u}$ to be the set of all positive definite correlation matrices, where $u$ denotes the number of sources. By definition, this type of matrices are real and symmetric. In addition, we introduce $\mathbb{P}(\cdot):\mathbb{R}^u\rightarrow [0,1]$ to denote the Gaussian measure. The results presented in this section and in \cite{Capone2022} can be combined to consider also the uncertainty of the lengthscales. 

In this work, we operate in the Bayesian setting, which is manifested in the following assumption:
\begin{assum}
The function \(\bm{f}(\cdot)\) is a sample of a Gaussian process with multi-task kernel \(\Sigma\, k(\cdot,\cdot):\mathcal{X}\times\mathcal{X} \longrightarrow \mathbb{R}^{u\times u}\) and hyperprior $p(\Sigma)$ of positive definite correlation matrices.
\label{assum:1}
\end{assum}
If a base kernel $k(\cdot,\cdot)$ is selected that satisfy the universal approximation property \cite{Micchelli2006}, then the multi-task kernel $\Sigma\,k(\cdot,\cdot)$ also satisfies the universal approximation property \cite{Caponnetto2008} which makes the assumption non-restrictive.
In addition, we assume that there exists a known scaling function $\beta: \mathcal{C}\rightarrow\mathbb{R}^+$ such that 
\begin{align}
\mathbb{P}\left(|\bm{f}(\bx)-\bm{\mu}_{\Sigma}(\bx)| \leq \beta^\frac{1}{2}(\Sigma)\bm{\sigma}_{\Sigma}(\bx),\quad \forall \bx \in\mathcal{X}\right) \geq 1-\rho,
\label{eq:safeness}
\end{align}
equivalently to Assumption 3.2 by \cite{Capone2022}. There exist different methods to define the scaling functions, e.g., \cite{Lederer2019}. If the input space $\mathcal{X}$ has finite cardinality, $\beta$ is constant and independent of the hyperparameters \cite{Snirivas2010}. 
In \eqref{eq:lengthscale_bounds}, the posterior probability density function is integrated over the interval from $\bm{\vartheta}'$ to $\bm{\vartheta}''$ which is selected such that some confidence is included. 
First, we need to define a relation for $\mathcal{C}$ to be able to order the members of the set.
Therefore, the function $h(\Sigma_1,\Sigma_2) = \max\mathrm{eig}\Sigma_1^{-1}\Sigma_2$ is introduced which maps two correlation matrices into the positive reals.
The reason for selecting this function follows from the definition of the error bounds in Lemma~\ref{lemma:gamma}. Using $h(\cdot,\cdot)$ we define the set
\begin{align}
\mathcal{C}_{\Sigma',\Sigma''} = \{\Sigma\in\mathcal{C}|h(\Sigma',\Sigma)\leq h(\Sigma',\Sigma'')\},
\label{eq:subset_sigma}
\end{align}
which is a subset of $\mathcal{C}$ and comprises all correlation matrices that have a smaller value in $h$ than two selected $(\Sigma',\Sigma'')$. Now we are able to define a set of bounding correlation matrices by
\begin{align*}
    \mathcal{C}_\delta=\Bigg\{(\Sigma',\Sigma'')\in\mathcal{C}^2\Bigg\vert\int_{\mathcal{C}_{\Sigma',\Sigma''}} p(\Sigma|\bm{X},\tilde{\bm{y}})\diff\Sigma \geq 1-\delta\Bigg\}.
\end{align*}
In words, we are searching for a bound $\Sigma',\Sigma''$ such that their corresponding set $\mathcal{C}_{\Sigma',\Sigma''}$ comprises enough members that establish a predefined confidence region on the posterior. 
After identifying a matching pair of bounds, we are able to modify the scaling function $\beta(\cdot)$ such that safeness for the uncertain hyperparameters can be guaranteed.
The extension is summarized in the following results.
\begin{lemma}
Let $\bm{\sigma}_{\Sigma'}(\bx),\bm{\sigma}_{\Sigma}(\bx)$ be the posterior variance conditioned on the data $\bm{\mathcal{D}}$ with different correlation matrices $\Sigma',\Sigma'' \in \mathcal{C}_\delta$, $\Sigma \in \mathcal{C}_{\Sigma',\Sigma''}$, and let  $\gamma^2 \geq h(\Sigma ',\Sigma'')$. Then
\begin{align*}
    \gamma^2\bm{\sigma}_{\Sigma'}(\bx)\quad \geq\quad \bm{\sigma}_{\Sigma}(\bx),\quad \forall \bx\in\mathcal{X}, \, \forall \Sigma\in\mathcal{C}_{\Sigma',\Sigma''}.
\end{align*}
\label{lemma:gamma}
\end{lemma}
\begin{proof}
See \version{\cref{appendix:proof_gamma}}.
\end{proof}

Lemma~\ref{lemma:gamma} bounds the posterior variance for all correlation matrices in the set $\mathcal{C}_{\Sigma',\Sigma''}$. However, bounding the posterior variance solely is not sufficient since the error bounds also depend on the posterior mean. This is addressed in the following Lemma~\ref{lemma:lambda}.

\begin{lemma} \label{lemma:lambda}
Let $\bm{\mu}$ be a member of the reproducing kernel Hilbert space (RKHS) $\mathcal{H}''$ with kernel $K''(\bx,\bx') = \Sigma''\,k(\bx,\bx')$, and let $\mathcal{H}'$ be an RKHS with kernel $K'(\bx,\bx') = \Sigma'\,k(\bx,\bx')$. Then with $\lambda^2 \geq h(\Sigma'',\Sigma')$ it follows
\begin{align*}
    \lambda^2||\bm{\mu}||_{K'}^2 \quad \geq \quad ||\bm{\mu}||_{K''}^2.
\end{align*}
\end{lemma}
\begin{proof}
See \version{\cref{appendix:proof_lambda}}.
\end{proof}

Lemma~\ref{lemma:lambda} is used to bound the posterior mean in different RKHS. In contrast to the unknown function $\bm{f}$ which is not restricted to the RKHS as stated in Assumption~\ref{assum:1}, the posterior mean is a member of the RKHS since $\bm{\mu}(\bx_*) = \langle \bm{\alpha}, K(\bx_*,\cdot)\rangle_K$ with $\bm{\alpha} = (\bm{K}_\Sigma+\sigma_n^2\,I)^{-1}\tilde{\bm{y}}$. 
With the two intermediate results robust error bounds in the multi-task setting can be specified as summarized in Theorem~\ref{theorem:betabar}.

\begin{theorem}
\label{theorem:betabar}
Let Assumption~\ref{assum:1} hold, and assume there exists a scaling function $\beta(\cdot)$ and a failure probability $\rho$ which are known such that \eqref{eq:safeness} holds. Let $(\Sigma',\Sigma'') \in \mathcal{C}_\delta$ with posterior of hyperparameters $p(\Sigma|\bm{X},\tilde{\bm{y}})$, let $\bm{\sigma}_{\Sigma'}^2$ be the posterior variance and $\bm{\mu}_{\Sigma_0}$ the posterior mean obtained with correlation matrices $\Sigma'$ and $\Sigma_0\in \mathcal{C}_{\Sigma',\Sigma''}$, respectively, and select $\gamma^2 \geq h(\Sigma',\Sigma'')$ and $\lambda^2 \geq \max_{\Sigma\in\mathcal{C}_{\Sigma',\Sigma''}}$\textcolor{orange}{$h(\Sigma,\Sigma')$\footnote{\textcolor{orange}{Typo: previously $h(\Sigma'',\Sigma)$, which is wrong because we want to bound $\lambda^2 ||\mu||_{K'}^2\geq ||\mu||_{K}^2\quad \forall \Sigma \in \mathcal{C}_{\Sigma',\Sigma''}$, cf. Lemma~\ref{lemma:lambda}}}}. Then with
\begin{align*}
    \bar{\beta} = \left(\lambda\frac{2\norm{\tilde{\bm{y}}}_2}{\sigma_n}+\gamma\max_{\Sigma \in\mathcal{C}_{\Sigma',\Sigma''}}\beta^\frac{1}{2}(\Sigma)\right)^2
\end{align*}
we have
\begin{align*}
    \mathbb{P}\left(\lvert \bm{f}(\bx)-\bm{\mu}_{\Sigma_0}(\bx)\rvert \leq \bar{\beta}^\frac{1}{2}\bm{\sigma}_{\Sigma'},\quad \forall \bx\in\mathcal{X}\right) \geq (1-\delta)(1-\rho).
\end{align*}
\end{theorem}
\begin{proof}
The result can be obtained by following the same steps as in Theorem~3.7 by \cite{Capone2022} where the scaling parameter $\gamma$ from Lemma~\ref{lemma:gamma} replaces Lemma~3.3, and $\lambda$ from Lemma~\ref{lemma:lambda} replaces Lemma~A.7 in \cite{Capone2022}. We take the maximum $\lambda$ over the set $\mathcal{C}_{\Sigma',\Sigma''}$ to ensure safeness for every $\Sigma_0\in\mathcal{C}_{\Sigma',\Sigma''}$. For $\gamma$ this follows immediately from the definition of the set $\mathcal{C}_{\Sigma',\Sigma''}$ in \eqref{eq:subset_sigma}.
\end{proof}
The correlation bounds $(\Sigma',\Sigma'')$ and scaling parameters $\gamma,\lambda$ can be computed numerically, for example, via comparing samples from the hyperparameter posterior which can be generated with MCMC methods. In general, the computational cost of inverting correlation matrices is lower when compared to computing the posterior distribution of a Gaussian process, as these matrices are notably smaller than the Gram matrix. The highest time consumption comes from computing the posterior distribution and generating posterior samples. Regarding the computation of posterior distributions, there are several approximation methods that can be employed to enhance the inference speed when dealing with a substantial amount of data points \cite{Snelson2005,titsias09a}.

\section{Safe Multi-Task Bayesian Optimization}
\label{sec:SaMSBO}
In safe Bayesian optimization, the goal is to minimize a multi-output objective $\bm{f}(\bx)$ while including constraints $\bm{g}(\bx)$ for each $f_i(\bx)$, i.e., $\min_{\bx\in\mathcal{X}} \bm{f}(\bx) \mathrm{\quad s.\,t.\quad} \bm{g}(\bx)\geq 0.$
Frequently $\bm{g}(\bx) = T-\bm{f}(\bx)$, where $T$ is a safety threshold. 
Nonetheless, all the results are applicable to arbitrary constraints that fulfill Assumption~\ref{assum:1}. Moreover, it is necessary to possess information regarding an initial safe set, denoted as $\mathcal{S}_0$ which comprises at least one input that fulfills the constraints. This assumption is widespread, as asserted by \cite{Kirschner2019,Sui2015}, even in practical scenarios where an initial estimation is typically provided. Note that since $\bm{f}(\bx)$ is unknown also $\bm{g}(\bx)$ is unknown according to the definition of the constraints and needs to be learned online. This means that only stochastic statements can be made about satisfying the constraints. Similarly to \texttt{SafeOpt-MC} in \cite{Berkenkamp2021}, the safe set is defined to be $\mathcal{S} \coloneqq \cap_{i=1}^u \{\bx\in\mathcal{X}|\mu_{\Sigma',i}(\bx)+\bar{\beta}^\frac{1}{2}\sigma_{\Sigma',i}(\bx) \leq T\}$, which includes inputs where the constraints are fulfilled with high probability according to Theorem~\ref{theorem:betabar}. Starting from the initial safe set, the domain is explored to find the minimum. There exist various approaches to conduct this optimization, in \cite{Berkenkamp2016} the most uncertain points in a subset of $\mathcal{S}$ are repetitively evaluated. In \cite{Sui2018} and \cite{Luebsen2023}, the exploration and exploitation phases are separated, meaning that the algorithm first focuses on expanding $\mathcal{S}$ and then on optimization. The commonality among all methods is their requirement for numerous evaluations of the objective function. This can be reduced by including simulations of the objective in the optimization. 
\begin{figure}[t]
    \centering
    \includegraphics[width=\textwidth]{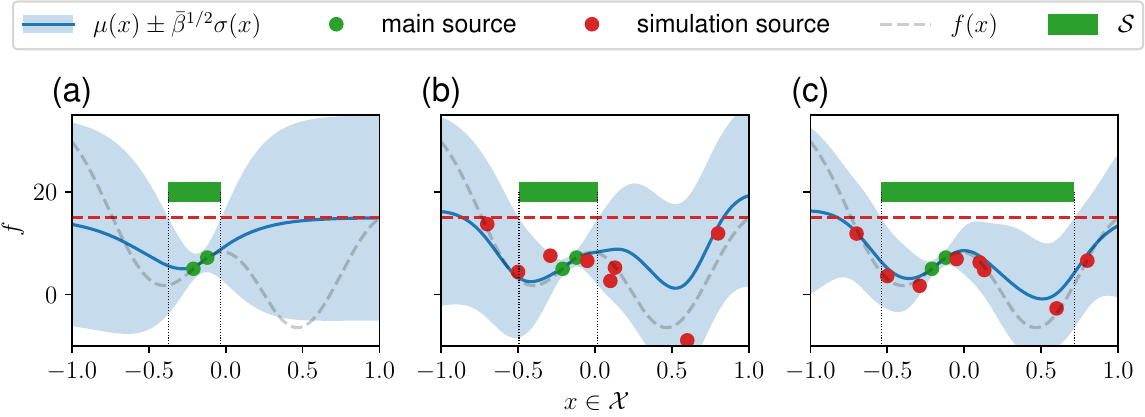}
    \caption{Overview of different safe Bayesian optimization settings with safety threshold $T$ denoted by "\textcolor{red}{\textbf{-\,-\,-}}". (a) shows the single-task setting, where no simulation samples are considered, and the safe region is the smallest. (b) visualizes the multi-task setting with slight correlation and (c) with high correlation. In both cases, The multi-task setting increases the safe region.}
    \label{fig:safebo_overview}
\end{figure}

\Cref{fig:safebo_overview} compares the safe region of the naive single-task optimization (a) with the multi-task setting where a low fidelity (b) and a high fidelity model (c) are considered. 
The fidelity of a model is reflected in the correlation matrix; the higher the agreement of model and the ground truth, the higher the information quality of the function value. \Cref{fig:safebo_overview} (b) shows that even small correlations extend the safe set.

The proposed method is summarized in \cref{algo:safemultibo}. We start with an initial safe set, a scaling function $\beta(\cdot)$ and a GP equipped with an appropriate hyperprior that reflects the initial guess about the hyperparameter. In the loop, posterior samples are generated and collected in the set $\mathcal{P}$. The samples allow to approximate the robust scaling parameter from Theorem~\ref{theorem:betabar}. In addition, $\Sigma'$ is loaded into the GP model. Then, the predictive distribution is used by an acquisition function to identify promising inputs at which the objective functions are evaluated, and both the GP model and the data set are updated. Finally, after the loop terminates, the best input of the main task is returned.

\IncMargin{1.5em}
\begin{algorithm2e}[t]
\DontPrintSemicolon
\KwIn{Initial safe set \(\mathcal{S}_0\), multi-task Gaussian process model $\mathcal{GP}$ with hyperpriors $p(\Sigma)$, scaling function $\beta(\cdot)$}
\KwOut{\(\bm{x}_\mathrm{opt}\)}
\While{termination condition not true}{
    $\mathcal{P} \leftarrow  \mathrm{MCMC}(\mathcal{GP},p(\Sigma))$  \tcp*[f]{\footnotesize generate samples via MCMC methods}\\ 
     $\bar{\beta}^{1/2}\leftarrow \mathcal{P}$, $\mathcal{GP}\leftarrow\Sigma'$ \tcp*[f]{\footnotesize Determine $\bar{\beta}$ from samples according to Theorem~\ref{theorem:betabar}}\\
     \tcp*[f]{\footnotesize load $\Sigma'$ into GP}\\
    $\mathcal{GP},\bm{\mathcal{D}}\leftarrow \mathrm{BO}(\mathcal{GP})$ \tcp*[f]{\footnotesize Perform BO step for all tasks and update GP model}\\ 
    
}
\(\bx_\mathrm{opt} \leftarrow \arg\,\min\bm{\mathcal{D}}\) \tcp*[f]{\footnotesize return best main task solution}\;
\caption{Safe Multi-Source Bayesian Optimization (\texttt{SaMSBO}) \label{algo:safemultibo}}
\end{algorithm2e}

\section{Simulation Results}
\label{sec:simulationresults}
In this section we demonstrate the proposed algorithm in simulation. The considered plant models the laser-based synchronization (LbSync) system \cite{Schulz2015} at European XFEL, akin to the system employed in \cite{Luebsen2023} and depicted in \cref{fig:lbsync}. The difference lies in the fact that all models $G$ represent the same laser models in this scenario, in other words, $G_i = G_j, \forall i,j=1,\dots,N$. The filter models $F_r$ and $F_{1:N}$ colorize the white Gaussian noise inputs $w_{1:N+1}$ to model environmental disturbances, e.g., vibrations, temperature changes and humidity. Throughout the optimization procedures, we operate under the assumption that, alongside the main task, there exist two supplementary tasks of information that simulate the main task, i.e., $\bm{f}(\bx) = [f_1(\bx),f_2(\bx),f_3(\bx)]^T$ where $f_1(\bx)$ is the main task.
In order to emulate the discrepancy between simulation and reality, the primary task employs nominal models, while supplementary tasks are subjected to disturbances. It is assumed that uncertainty resides in the filter models, given that the laser model can be accurately identified and exhibits minimal variation over time, while disturbance sources change more frequently. Furthermore, it is essential to ensure the safety of the primary task i.e., $\bm{g}(\bx)=[T-f_1(\bx),0,0]^T$, as the supplementary tasks are simulations, where $T=30$. Hence, evaluating unsafe regions will not damage the system, rather, this can help to estimate the safe region.
\begin{figure}[t]
    \centering
    \includegraphics[width = \textwidth]{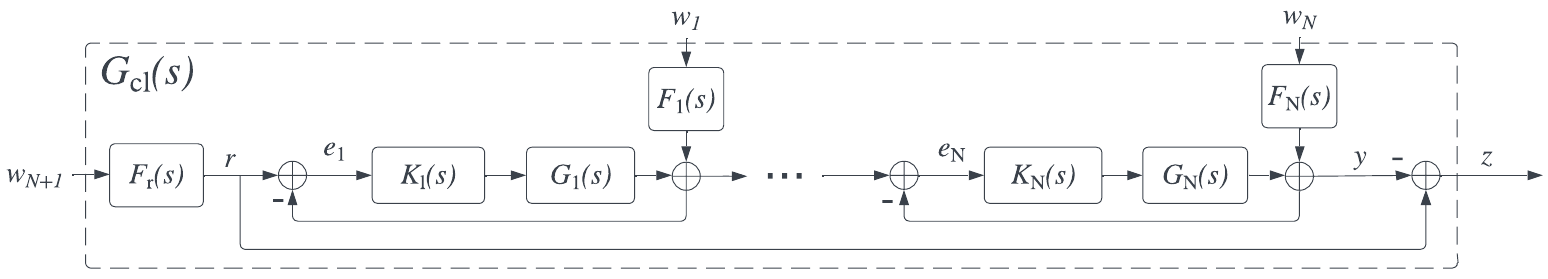}
    \caption{Illustration of the interconnected system. The blocks $F_r$ and $F_i, i=1,\dots,N$ denote disturbance filters which colorize the white noise inputs $w_j, j=1,\dots,N+1$. $G_i$ denote the laser plants and $K_i$ PI controllers for each subsystem.}
    \label{fig:lbsync}
\end{figure}
The goal is to minimize the root-mean-square seminorm of the performance output $z$ by tuning the PI parameters of the controllers $K_{i:N}$. Following Parseval's Theorem this corresponds to an $H_2$ minimization of the closed-loop system \cite{Heuer2018}.
Implementations are carried out using GPyTorch \cite{gardner2018gpytorch} and BoTorch \cite{balandat2020botorch}, and MCMC samples are generated with the No-U Turn Sampler algorithm \cite{Hoffman2014}. In addition, a Lewandowski-Kurowicka-Joe distribution \cite{Lewandowski2009} is used for prior distribution over correlation matrices, because it easily allows including prior knowledge about the expected correlation by adjusting the shape factor $\eta\in(0,\infty)$. Setting $\eta < 1$ matrices with higher correlation are favored, while for $\eta > 1$ low correlation matrices are favored. Note that, the definitions in Lemma~\ref{lemma:gamma},~\ref{lemma:lambda} provide error bounds to ensure safeness along all tasks, while in this setting, only the main task needs to be safe. Therefore, we neglect the influence of the uncertainty of the posterior mean in Theorem~\ref{theorem:betabar}, i.e., $\lambda^2 = 0$, to avoid overly conservative error bounds in the optimization. In addition, we select a constant scaling factor $\beta = 4$ and adjust the remaining hyperparameters such that the single-task optimization is safe.
Per BO step, one evaluation of the main task and 15 evaluations of the supplementary tasks are taken. 

\Cref{fig:comparison} summarizes the benchmark results. In (a), the robustness of the algorithm is investigated, where the supplementary tasks are constructed with disturbed filter transfer functions. The disturbances are generated by sampling from a uniform distribution with magnitude according to the line color in the legend. A $\pm10\%$ variation signifies that the state space model values can fluctuate by up to $10\%$ (in both directions) from their nominal values. (a) displays the average of the best observation from 20 instances of the optimization. For each instance, the filter models' disturbances are resampled. Clearly, the number of observations increases with the uncertainty of the supplementary tasks because the optimal solutions among the tasks may not perfectly match. Nevertheless, the optimal solutions are found for all uncertainties. In (b) the initial points are changing, and the models are fixed across the iterations. Additionally, the outcomes obtained by applying a Safe BO algorithm, similar to \texttt{SafeOpt} \cite{Sui2015} but with EI acquisition function, and \texttt{MoSaOpt}, a line search method from \cite{Luebsen2023}, are plotted. The mean values are denoted by the lines and the shaded area represents the standard deviation. \texttt{SaMSBO} outperforms both regarding solution quality and sample efficiency of the main task.
\begin{figure}[t]
    \centering
    \includegraphics[width = \linewidth]{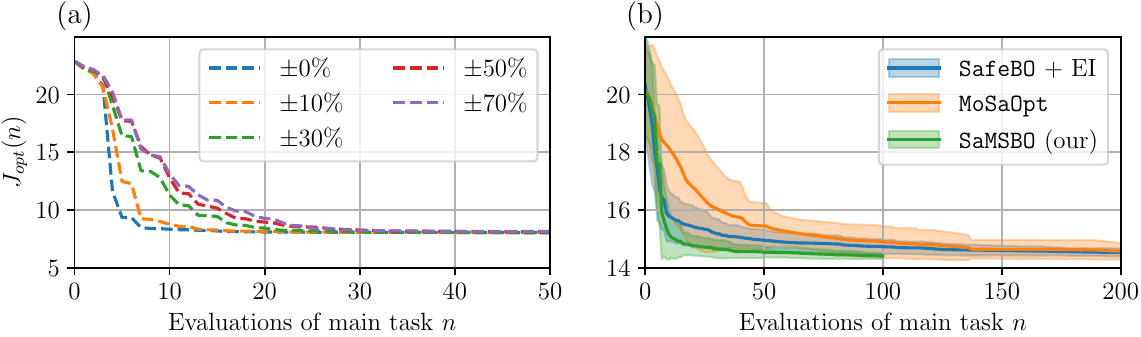}
    \caption{(a) shows the performance of \texttt{SaMSBO} by tuning a chain for $N = 2$ lasers, where the extra tasks have disturbed filter transfer functions. The line colors indicate the range of the disturbance. (b) shows the performance of different BO algorithms applied on a chain with $N = 5$ lasers. In this trial the filter disturbance lies in the range $\pm10\%$.}
    \label{fig:comparison}
\end{figure}

\section{Conclusion and Outlook}
We proposed the first robustly safe Bayesian optimization algorithm in a multi-task setting. We theoretically derived bounds that guarantee safeness for unknown correlation hyperparameters with high probability. Moreover, the proposed algorithm was benchmarked against other state-of-the-art methods in simulation. \texttt{SaMSBO} demonstrated superior solution quality and sample efficiency, ultimately achieving the most favorable convergence rate. However, one drawback is the doubled computation time per iteration compared to naive multi-task BO, due to the MCMC approximation of the hyper posterior, especially in later stages of the algorithm. Nevertheless, based on the simulation results, the total optimization time is still significantly reduced overall.

In the future, numerous aspects can be enhanced. Primarily, the scaling bounds specified in Lemma~\ref{lemma:gamma},~\ref{lemma:lambda} will become less stringent if safety is required for the main task solely, which is in real applications usually the case. Furthermore, including an efficient exploration technique of the safe set is crucial, particularly in high-dimensional scenarios. An approach to tackle this challenge has been introduced by \cite{Zagorowska2023}, wherein exploration is reformulated as an optimization problem which can be efficiently solved.
Finally, our plan involves testing the algorithm in a real environment by optimizing the LbSync system at European XFEL.

\acks{We acknowledge support from Deutsches Elektronen-Synchrotron DESY Hamburg, Germany, a member of the Helmholtz Association HGF. \copyright\ All figures and pictures under a CC BY 4.0 license.}

\bibliography{biblio.bib}

\begin{thebibliography}{32}
\providecommand{\natexlab}[1]{#1}
\providecommand{\url}[1]{\texttt{#1}}
\expandafter\ifx\csname urlstyle\endcsname\relax
  \providecommand{\doi}[1]{doi: #1}\else
  \providecommand{\doi}{doi: \begingroup \urlstyle{rm}\Url}\fi

\bibitem[{{\'A}}lvarez and Lawrence(2011)]{Alvarez2011}
Mauricio~A. {{\'A}}lvarez and Neil~D. Lawrence.
\newblock Computationally efficient convolved multiple output {Gaussian}
  processes.
\newblock \emph{Journal of Machine Learning Research}, 12\penalty0
  (41):\penalty0 1459--1500, 2011.

\bibitem[Balandat et~al.(2020)Balandat, Karrer, Jiang, Daulton, Letham, Wilson,
  and Bakshy]{balandat2020botorch}
Maximilian Balandat, Brian Karrer, Daniel~R. Jiang, Samuel Daulton, Benjamin
  Letham, Andrew~Gordon Wilson, and Eytan Bakshy.
\newblock Botorch: A framework for efficient monte-carlo {Bayesian}
  optimization.
\newblock In \emph{Advances in Neural Information Processing Systems 33}, 2020.

\bibitem[Berkenkamp et~al.(2016)Berkenkamp, Schoellig, and
  Krause]{Berkenkamp2016}
Felix Berkenkamp, Angela~P. Schoellig, and Andreas Krause.
\newblock Safe controller optimization for quadrotors with {Gaussian}
  processes.
\newblock In \emph{IEEE Int. Conf. Robot. Autom (ICRA)}, pages 491--496, 2016.

\bibitem[Berkenkamp et~al.(2021)Berkenkamp, Krause, and
  Schoellig]{Berkenkamp2021}
Felix Berkenkamp, Andreas Krause, and Angela~P. Schoellig.
\newblock {Bayesian} optimization with safety constraints: safe and automatic
  parameter tuning in robotics.
\newblock \emph{Machine Learning}, 2021.
\newblock ISSN 1573-0565.

\bibitem[Bonilla et~al.(2007)Bonilla, Chai, and Williams]{Bonilla2007}
Edwin~V Bonilla, Kian Chai, and Christopher Williams.
\newblock Multi-task {Gaussian} process prediction.
\newblock In \emph{Advances in Neural Information Processing Systems},
  volume~20. Curran Associates, Inc., 2007.

\bibitem[Capone et~al.(2022)Capone, Lederer, and Hirche]{Capone2022}
Alexandre Capone, Armin Lederer, and Sandra Hirche.
\newblock {Gaussian} process uniform error bounds with unknown hyperparameters
  for safety-critical applications.
\newblock In \emph{Proceedings of the 39th International Conference on Machine
  Learning}, 2022.

\bibitem[Caponnetto et~al.(2008)Caponnetto, Micchelli, Pontil, and
  Ying]{Caponnetto2008}
Andrea Caponnetto, Charles~A. Micchelli, Massimiliano Pontil, and Yiming Ying.
\newblock Universal multi-task kernels.
\newblock \emph{Journal of Machine Learning Research}, 9\penalty0
  (52):\penalty0 1615--1646, 2008.

\bibitem[Chowdhury and Gopalan(2017)]{chowdhury17a}
Sayak~Ray Chowdhury and Aditya Gopalan.
\newblock On kernelized multi-armed bandits.
\newblock In Doina Precup and Yee~Whye Teh, editors, \emph{Proceedings of the
  34th International Conference on Machine Learning}, volume~70 of
  \emph{Proceedings of Machine Learning Research}, pages 844--853. PMLR, 06--11
  Aug 2017.
\newblock URL \url{https://proceedings.mlr.press/v70/chowdhury17a.html}.

\bibitem[Ferran~Pousa et~al.(2023)Ferran~Pousa, Jalas, Kirchen, Martinez de~la
  Ossa, Th\'evenet, Hudson, Larson, Huebl, Vay, and Lehe]{Pousa2023}
A.~Ferran~Pousa, S.~Jalas, M.~Kirchen, A.~Martinez de~la Ossa, M.~Th\'evenet,
  S.~Hudson, J.~Larson, A.~Huebl, J.-L. Vay, and R.~Lehe.
\newblock {Bayesian} optimization of laser-plasma accelerators assisted by
  reduced physical models.
\newblock \emph{Phys. Rev. Accel. Beams}, 26:\penalty0 084601, Aug 2023.
\newblock \doi{10.1103/PhysRevAccelBeams.26.084601}.

\bibitem[Fiedler et~al.(2021)Fiedler, Scherer, and Trimpe]{fiedler21}
Christian Fiedler, Carsten Scherer, and Sebastian Trimpe.
\newblock Practical and rigorous uncertainty bounds for gaussian process
  regression.
\newblock 05 2021.

\bibitem[Gardner et~al.(2018)Gardner, Pleiss, Bindel, Weinberger, and
  Wilson]{gardner2018gpytorch}
Jacob~R Gardner, Geoff Pleiss, David Bindel, Kilian~Q Weinberger, and
  Andrew~Gordon Wilson.
\newblock Gpytorch: Blackbox matrix-matrix {Gaussian} process inference with
  gpu acceleration.
\newblock In \emph{Advances in Neural Information Processing Systems}, 2018.

\bibitem[Heuer(2018)]{Heuer2018}
Michael Heuer.
\newblock \emph{Identification and control of the laser-based synchronization
  system for the {European X-ray Free Electron Laser}}.
\newblock Doctoral dissertation, Technische Universität Hamburg-Harburg, 2018.

\bibitem[Hoffman and Gelman(2014)]{Hoffman2014}
Matthew~D. Hoffman and Andrew Gelman.
\newblock The no-u-turn sampler: adaptively setting path lengths in
  {Hamiltonian} {Monte Carlo}.
\newblock \emph{J. Mach. Learn. Res.}, pages 1--30, 2014.

\bibitem[Horn(2017)]{horn2017}
Roger~A. Horn.
\newblock \emph{Matrix analysis}.
\newblock Cambridge University Press, New York, NY, second edition, corrected
  reprint edition, 2017.
\newblock ISBN 9780521548236.

\bibitem[Jones et~al.(1998)Jones, Schonlau, and Welch]{Jones1998}
Donald~R. Jones, Matthias Schonlau, and William~J. Welch.
\newblock Efficient global optimization of expensive black-box functions.
\newblock \emph{Journal of Global Optimization}, 13\penalty0 (4):\penalty0
  455--492, Dec 1998.
\newblock ISSN 1573-2916.

\bibitem[Kirschner et~al.(2019)Kirschner, Mutny, Hiller, Ischebeck, and
  Krause]{Kirschner2019}
Johannes Kirschner, Mojmir Mutny, Nicole Hiller, Rasmus Ischebeck, and Andreas
  Krause.
\newblock Adaptive and safe {Bayesian} optimization in high dimensions via
  one-dimensional subspaces.
\newblock In \emph{36th Int. Conf. Mach. Learn. (ICML)}, pages 3429--3438,
  2019.

\bibitem[Lederer et~al.(2019)Lederer, Umlauft, and Hirche]{Lederer2019}
Armin Lederer, Jonas Umlauft, and Sandra Hirche.
\newblock Uniform error bounds for {Gaussian} process regression with
  application to safe control.
\newblock In \emph{Advances in Neural Information Pro-cessing Systems}, page
  659–669, June 2019.

\bibitem[Letham and Bakshy(2019)]{Letham2019}
Benjamin Letham and Eytan Bakshy.
\newblock {Bayesian} optimization for policy search via
  online-offlineexperimentation.
\newblock \emph{J. Mach. Learn. Res.}, pages 593--1623, 2019.

\bibitem[Lewandowski et~al.(2009)Lewandowski, Kurowicka, and
  Joe]{Lewandowski2009}
Daniel Lewandowski, Dorota Kurowicka, and Harry Joe.
\newblock Generating random correlation matrices based on vines and extended
  onion method.
\newblock \emph{Journal of Multivariate Analysis}, 100\penalty0 (9):\penalty0
  1989--2001, 2009.
\newblock ISSN 0047-259X.
\newblock \doi{https://doi.org/10.1016/j.jmva.2009.04.008}.

\bibitem[Lübsen et~al.(2023)Lübsen, Schütte, Schulz, and
  Eichler]{Luebsen2023}
Jannis~O. Lübsen, Maximilian Schütte, Sebastian Schulz, and Annika Eichler.
\newblock A safe {Bayesian} optimization algorithm for tuning the optical
  synchronization system at {European XFEL}.
\newblock In \emph{22th World Congr. Int. Fed. Autom. Control (IFAC)}, 2023.

\bibitem[Micchelli et~al.(2006)Micchelli, Xu, and Zhang]{Micchelli2006}
Charles~A. Micchelli, Yuesheng Xu, and Haizhang Zhang.
\newblock Universal kernels.
\newblock \emph{Journal of Machine Learning Research}, 7\penalty0
  (95):\penalty0 2651--2667, 2006.

\bibitem[Murphy(2012)]{murphy2021}
Kevin~P Murphy.
\newblock \emph{Machine learning: a probabilistic perspective}.
\newblock Cambridge, MA, 2012.

\bibitem[Schulz et~al.(2015)Schulz, Grguras, Behrens, Bromberger, Costello,
  Czwalinna, Felber, Hoffmann, Ilchen, Liu, Mazza, Meyer, Pfeiffer, Predki,
  Schefer, Schmidt, Wegner, Schlarb, and Cavalieri]{Schulz2015}
S~Schulz, Ivanka Grguras, C~Behrens, Hubertus Bromberger, John Costello, Marie
  Czwalinna, Matthias Felber, M~Hoffmann, M~Ilchen, H~Liu, Tommaso Mazza,
  M~Meyer, Sven Pfeiffer, Pawel Predki, S~Schefer, Christian Schmidt, U~Wegner,
  H.~Schlarb, and A~Cavalieri.
\newblock Femtosecond all-optical synchronization of an {X-ray} free-electron
  laser.
\newblock \emph{Nature communications}, 6:\penalty0 5938, 01 2015.

\bibitem[Snelson and Ghahramani(2005)]{Snelson2005}
Edward Snelson and Zoubin Ghahramani.
\newblock Sparse gaussian processes using pseudo-inputs.
\newblock In Y.~Weiss, B.~Sch\"{o}lkopf, and J.~Platt, editors, \emph{Advances
  in Neural Information Processing Systems}, volume~18. MIT Press, 2005.
\newblock URL
  \url{https://proceedings.neurips.cc/paper_files/paper/2005/file/4491777b1aa8b5b32c2e8666dbe1a495-Paper.pdf}.

\bibitem[Srinivas et~al.(2010)Srinivas, Krause, Kakade, and
  Seeger]{Snirivas2010}
Niranjan Srinivas, Andreas Krause, Sham Kakade, and Matthias Seeger.
\newblock {Gaussian} process optimization in the bandit setting: No regret and
  experimental design.
\newblock pages 1015--1022, 07 2010.

\bibitem[Sui et~al.(2015)Sui, Gotovos, Burdick, and Krause]{Sui2015}
Yanan Sui, Alkis Gotovos, Joel Burdick, and Andreas Krause.
\newblock Safe exploration for optimization with {Gaussian} processes.
\newblock In \emph{32nd Int. Conf. Mach. Learn. (ICML)}, volume~37 of
  \emph{Proceedings of Machine Learning Research}, pages 997--1005, Lille,
  France, 07--09 Jul 2015. PMLR.

\bibitem[Sui et~al.(2018)Sui, Zhuang, Burdick, and Yue]{Sui2018}
Yanan Sui, Vincent Zhuang, Joel~W. Burdick, and Yisong Yue.
\newblock Stagewise safe {Bayesian} optimization with {Gaussian} processes.
\newblock In \emph{35th Int. Conf. Mach. Learn. (ICML)}, 2018.

\bibitem[Sun et~al.(2021)Sun, Khojasteh, Shekhar, and Fan]{Sun2021}
Dawei Sun, Mohammad~Javad Khojasteh, Shubhanshu Shekhar, and Chuchu Fan.
\newblock Uncertain-aware safe exploratory planning using {Gaussian} process
  and neural control contraction metric.
\newblock In \emph{Proceedings of the 3rd Conference on Learning for Dynamics
  and Control}, volume 144 of \emph{Proceedings of Machine Learning Research},
  pages 728--741. PMLR, 07 -- 08 June 2021.

\bibitem[Swersky et~al.(2013)Swersky, Snoek, and Adams]{Swersky2013}
Kevin Swersky, Jasper Snoek, and Ryan~P. Adams.
\newblock Multi-task {Bayesian} optimization.
\newblock In \emph{Advances in Neural Information Processing Systems}, 2013.

\bibitem[Titsias(2009)]{titsias09a}
Michalis Titsias.
\newblock Variational learning of inducing variables in sparse gaussian
  processes.
\newblock In David van Dyk and Max Welling, editors, \emph{Proceedings of the
  Twelth International Conference on Artificial Intelligence and Statistics},
  volume~5 of \emph{Proceedings of Machine Learning Research}, pages 567--574,
  Hilton Clearwater Beach Resort, Clearwater Beach, Florida USA, 16--18 Apr
  2009. PMLR.
\newblock URL \url{https://proceedings.mlr.press/v5/titsias09a.html}.

\bibitem[Williams and Rasmussen(2006)]{Williams2006}
Christopher~K Williams and Carl~Edward Rasmussen.
\newblock \emph{{Gaussian} processes for machine learning}, volume~2.
\newblock MIT press Cambridge, MA, 2006.

\bibitem[Zagorowska et~al.(2023)Zagorowska, Balta, Behrunani, Rupenyan, and
  Lygeros]{Zagorowska2023}
Marta Zagorowska, Efe~C. Balta, Varsha Behrunani, Alisa Rupenyan, and John
  Lygeros.
\newblock Efficient sample selection for safe learning.
\newblock In \emph{22th World Congr. Int. Fed. Autom. Control (IFAC)}, 2023.

\end{thebibliography}

\ifthenelse{\equal{\isarxiv}{true}}{\appendix
\section{Proof of Lemma~\ref{lemma:gamma}}
\label{appendix:proof_gamma}
 \begin{lemma}[Schur Product Theorem \cite{horn2017}]
For two positive semidefinite matrices $A,B$ the Hadamard product is positive semidefinite, i.e., $A,B \geq 0$
implies $A \circ B \geq 0$ \label{lemma:schur_product}
\label{lemma:hadamard_prod}
\end{lemma}

\begin{lemma}
Let $\bm{\Sigma}$ denote a matrix of the form
\begin{align*}
    \bm{\Sigma} = \begin{bmatrix}
    \sigma_{1,1}^2\mathds{1}_{N_1}\mathds{1}_{N_1}^T & \dots & \sigma_{1,u}^2\mathds{1}_{N_1}\mathds{1}_{N_u}^T\\
    \vdots & \ddots & \vdots \\
    \sigma_{u,1}^2\mathds{1}_{N_u}\mathds{1}_{N_1}^T & \dots & \sigma_{u,u}^2\mathds{1}_{N_u}\mathds{1}_{N_u}^T 
    \end{bmatrix},
\end{align*}
where $\mathds{1}_N$ is a column vector of ones with length $N$. Then, $\bm{\Sigma}$ is positive semidefinite if 
\begin{align*}
\begin{bmatrix}\sigma_{1,1}^2 & \dots & \sigma_{1,u}^2 \\
\vdots & \ddots & \vdots \\
\sigma_{u,1}^2 & \dots & \sigma_{u,u}^2\end{bmatrix} \geq 0.
\end{align*}
\label{lemma:decomposition}
\end{lemma}

\begin{proof}
The singular value decomposition of a vector $\mathds{1}_N$ can be written as $\mathds{1}_N = UE1$. The first column of $U$ corresponds to a vector of ones scaled to the unit length and the remaining columns are filled with orthonormal vectors. The only nonzero singular value is placed at the top of $E$ with $E_{1} = \sqrt{N}$ and the remaining entries are zero. Hence, we can rewrite $\bm{\Sigma}$ in the form
\begin{align*}
    \bm{\Sigma} = \begin{bmatrix}
    \sigma_{1,1}^2 U_1 E_1 E_1^T U_1^T & \dots & \sigma_{1,u}^2 U_1 E_1 E_u^T U_u^T\\
    \vdots & \ddots & \vdots \\
    \sigma_{u,1}^2\sigma_{u,1}^2 U_u E_u E_1^T U_1^T & \dots & \sigma_{u,u}^2U_u E_u E_u^T U_u^T
    \end{bmatrix}.
\end{align*}
    
Furthermore, by defining $\bm{U} = \mathrm{blkdiag}([U_i]_{i=1}^u)$, $\bm{\Sigma}$ can be decomposed, such that

\begin{align*}
\bm{U}\bm{E}_\Sigma\bm{U}^T= \bm{U} \begin{bmatrix}
\begin{bmatrix}
\sigma_{1,1}^2N_{11} & 0 & \cdots & 0 \\
0 & 0 & \cdots & 0 \\
\vdots & \vdots & \ddots & \vdots \\
0 & 0 & \cdots & 0 \\
\end{bmatrix} & \dots & \begin{bmatrix}
\sigma_{1,u}^2N_{u1} & 0 & \cdots & 0 \\
0 & 0 & \cdots & 0 \\
\vdots & \vdots & \ddots & \vdots \\
0 & 0 & \cdots & 0 \\
\end{bmatrix}\\
    \vdots & \ddots & \vdots \\
    \begin{bmatrix}
\sigma_{u,1}^2N_{1u} & 0 & \cdots & 0 \\
0 & 0 & \cdots & 0 \\
\vdots & \vdots & \ddots & \vdots \\
0 & 0 & \cdots & 0 \\
\end{bmatrix} & \dots & \begin{bmatrix}
\sigma_{u,u}^2N_{uu} & 0 & \cdots & 0 \\
0 & 0 & \cdots & 0 \\
\vdots & \vdots & \ddots & \vdots \\
0 & 0 & \cdots & 0 \\
\end{bmatrix}
    \end{bmatrix}  \bm{U}^T,
\end{align*}
where $N_{ij} = \sqrt{N_iN_j}$ is the product of the singular values. We define a permutation matrix that swaps the rows and columns, such that
\begin{align}
    P^T\bm{E}_\Sigma P = \begin{bmatrix}
\begin{array}{c|c}
\begin{matrix}\sigma_{1,1}^2N_{11} & \dots & \sigma_{1,u}^2N_{1u} \\
\vdots & \ddots & \vdots \\
\sigma_{u,1}^2N_{u1} & \dots & \sigma_{u,u}^2N_{uu}\end{matrix} & 0 \\ \hline
0 & 0
\end{array}
\end{bmatrix}.
\label{eq:permutation_applied}
\end{align}
 Since, for every permutation matrix, we have $P^T = P^{-1}$ \cite{horn2017} meaning that \eqref{eq:permutation_applied} is a similarity transformation. We can conclude that $\bm{\Sigma}$ is only positive semidefinite if and only if the upper left part of \eqref{eq:permutation_applied} is positive semidefinite.
 
Observe that the upper left matrix can be reformulated as a Hadamard product 
 \begin{align*}
    T\circ\Sigma = \begin{bmatrix}N_{11} & \dots & N_{1u} \\
\vdots & \ddots & \vdots \\
N_{u1} & \dots & N_{uu}\end{bmatrix} \circ \begin{bmatrix}\sigma_{1,1}^2 & \dots & \sigma_{1,u}^2 \\
\vdots & \ddots & \vdots \\
\sigma_{u,1}^2 & \dots & \sigma_{u,u}^2\end{bmatrix} \geq 0.
 \end{align*}
 Using Lemma~\ref{lemma:schur_product} and observe that $T$ is always positive semidefinite, we see that $\Sigma$ need to be positive semidefinite.
\end{proof}

Now, with Lemma~\ref{lemma:hadamard_prod} and Lemma~\ref{lemma:decomposition} we are able to proof Lemma~\ref{lemma:gamma}.
\begin{proof}[Lemma~\ref{lemma:gamma}]
From Lemma A.4 in \cite{Capone2022}, we know that it suffices to show that 
$\gamma^2 \bm{K}_{\Sigma'} \geq \bm{K}_{\Sigma}$ where $\bm{K}_{\Sigma},\bm{K}_{\Sigma'}$ are the Gram matrices of the joint distribution.
We assume to have for each of the $u$ tasks $N_i$ test and data points. Then, the Gram matrix is 
\begin{align*}
\bm{K}_{\Sigma}= 
\begin{bmatrix} 
\sigma_{1,1}^2K_{1,1}& \dots & \sigma_{1,u}^2K_{1,u}\\
\vdots & \ddots & \vdots \\
\sigma_{u,1}^2K_{u,1} & \dots & \sigma_{u,u}^2K_{u,u}
\end{bmatrix},
\end{align*}
where the matrices $K_{i,j} \in \mathbb{R}^{N_i\times N_j}$ and $\sigma_{i,j}^2$ are scalar. The equation can be rewritten as a Hadamard product
\begin{align*}
    \bm{K}_\Sigma = \bm{\Sigma} \circ \bar{\bm{K}} = \begin{bmatrix} 
\Sigma_{1,1} & \dots & \Sigma_{1,u}\\
\vdots & \ddots & \vdots \\
\Sigma_{u,1} & \dots & \Sigma_{u,u}
\end{bmatrix} 
\circ
\begin{bmatrix} 
K_{1,1}& \dots & K_{1,u}\\
\vdots & \ddots & \vdots \\
K_{u,1} & \dots & K_{u,u}
\end{bmatrix}, 
\end{align*}

where $\Sigma_{i,j} = \sigma_{i,j}^2 \mathds{1}_{N_i}\mathds{1}_{N_j}^T$.
 With the decomposition into the Hadamard product we have
\begin{align*}
    (\gamma^2\bm{\Sigma}'-\bm{\Sigma})\circ \bar{\bm{K}} \geq 0.
\end{align*}
According to Lemma~\ref{lemma:schur_product} positive semidefiniteness is preserved if we show that both $\gamma^2\bm{\Sigma}'-\bm{\Sigma}$ and $\bar{\bm{K}}$ are positive semidefinite. Clearly, this is true by definition for $\bar{\bm{K}}$ since this is the Gram matrix of a covariance function. Hence, it suffices to show that $\gamma^2\bm{\Sigma}'-\bm{\Sigma}$ is positive semidefinite.
Using Lemma~\ref{lemma:decomposition} this is true if $\gamma^2 \geq \max\mathrm{eig}(\Sigma'^{-1}\Sigma)$. By the definition of the set $\mathcal{C}_{\Sigma',\Sigma''}$ this holds for all $\Sigma \in \mathcal{C}_{\Sigma',\Sigma''}$ if 
\begin{align*}
\gamma^2 \geq h(\Sigma',\Sigma'') = \max\mathrm{eig}(\Sigma'^{-1}\Sigma'').
\end{align*}
\end{proof}


\section{Proof of Lemma~\ref{lemma:lambda}}
\label{appendix:proof_lambda}
\begin{proof}[Lemma~\ref{lemma:lambda}]
Note that both RKHS have the same base kernel $k(\cdot,\cdot)$ and the only difference lies in the multiplication of different correlation matrices. Hence, we need to find a formulation for $||\cdot||_K$ in terms of $||\cdot||_k$. Let $\bm{\mu}(\bx) = B\, \bm{u}(\bx)$, $\bm{\mu}\in\mathcal{H}_K$ with kernel $K(\bx,\bx') = \Sigma\, k(\bx,\bx')$ where $\Sigma = BB^T > 0$, and latent functions $\bm{u}=[u_i]_{i=1}^u, u_i\in\mathcal{H}_k$, with kernel $k(\bx,\bx')$. First, we are seeking for basis functions in terms of $k(\cdot,\cdot)$ that construct $K(\cdot,\cdot)$. Let $K(\cdot,\bx) = B^T \otimes k(\cdot,\bx)$, we have

\begin{center}
\begin{tabular}{C C L C L }
      K(\bx,\bx') & = & \langle K(\bx,\cdot),K(\bx',\cdot)\rangle_K & = & \langle B^T\otimes k(\bx,\cdot),B^T\otimes k(\bx',\cdot)\rangle_k\\
     & = & BB^T \otimes k(\bx,\bx') & = & \Sigma\, k(\bx,\bx').
\end{tabular}
\end{center}
After validating the basis functions, we need to find a formulation for $\bm{\mu}(\bx)$ in the space $\mathcal{H}_K$. With the reproducing property $\bm{u}(\bx)=\langle\bm{u},k(\bx,\cdot)\rangle_k$, we have

\begin{center}
\begin{tabular}{C C L C L}
     \bm{\mu}(\bx) & = &  B\, \bm{u}(\bx)
     & = &   \langle k(\bx,\cdot),\bm{u}B^T\rangle_k \\
    &=& \langle k(\bx,\cdot), \bm{\mu}_k\rangle_k  &=&
     \langle \mathrm{vec}(\bm{\mu}_k),(B^\dag)^T B^T\otimes k(\bx,\cdot)\rangle_k\\
    &=&  \langle \mathrm{vec}(\bm{\mu}_k(B^\dag)^T),B^T\otimes k(\bx,\cdot)\rangle_k 
    &=&  \langle \bm{\mu}_K,K(\bx,\cdot)\rangle_K,
\end{tabular}
\end{center}
where $B^\dag$ is the pseudoinverse of $B$. We used here the mixed Kronecker matrix-vector product, $\mathrm{vec}(ABC^T) = (C\otimes A)\mathrm{vec}(B)$.
Thus, the norm in $\mathcal{H}_K$ can be rewritten as

\begin{center}
    \begin{tabular}{C C L C L}
         ||\bm{\mu}_K||_K^2 &=& \langle\bm{\mu}_K,\bm{\mu}_K\rangle_K \nonumber & = & \langle\mathrm{vec}(\bm{\mu}_k(B^\dag)^T),\mathrm{vec}(\bm{\mu}_k(B^\dag)^T)\rangle_k\\
    & = & \langle(B^\dag\otimes I)\mathrm{vec}(\bm{\mu}_k),(B^\dag\otimes I)\mathrm{vec}(\bm{\mu}_k)\rangle_k & = & \langle(\Sigma^{-1}\otimes I)\mathrm{vec}(\bm{\mu}_k),\mathrm{vec}(\bm{\mu}_k)\rangle_k\\
    & = & \sum_{i,j}^u \Sigma_{ij}^{-1}\langle \mu_i,\mu_j\rangle_k \nonumber & = & \mathds{1}^T (\Sigma^{-1} \circ M) \mathds{1}.
    \end{tabular}
\end{center}
Now we are able to compare the norms in both spaces:
\begin{alignat*}{3}
     & \lambda^2||\bm{\mu}||_{K'}^2 \quad &\geq& \quad ||\bm{\mu}||_{K''}^2 \nonumber\\
   \Leftrightarrow& \lambda^2 \Sigma'^{-1} \circ M \quad &\geq& \quad \Sigma''^{-1} \circ M\nonumber\\
   \Leftrightarrow&  (\lambda^2 \Sigma'^{-1}-\Sigma''^{-1}) \circ M \quad &\geq& \quad 0.  \nonumber
\end{alignat*}
Since $M$ is positive semidefinite and applying Lemma~\ref{lemma:schur_product}, we obtain
\begin{align*}
    \lambda^2  \geq h(\Sigma'',\Sigma') =  \max\mathrm{eig}\Sigma''^{-1}\Sigma'.
\end{align*}


\end{proof}


}{}
\end{document}